\documentclass[conference, 11pt]{IEEEtran}
\IEEEoverridecommandlockouts

\usepackage{subcaption}
\usepackage{amsmath}

\usepackage{amsthm}
\usepackage{bbm}
\usepackage{amssymb}
\usepackage{algorithm}
\usepackage[noend]{algorithmic}
\usepackage{xcolor}
\usepackage{mathtools}
\DeclareMathOperator*{\argmax}{argmax}

\newtheorem{theorem}{Theorem}
\theoremstyle{definition}
\newtheorem{definition}{Definition}

\usepackage{fancyhdr}
\usepackage{url}

\title{
Self-Supervised Online Reward Shaping \\ in Sparse-Reward Environments
}

\author{Farzan Memarian$^*$$^\dagger$, Wonjoon Goo$^*$, Rudolf Lioutikov, Scott Niekum, and Ufuk Tocpu \\ University of Texas at Austin, TX, USA
\thanks{* Equal Contribution. }
\thanks{$\dagger$ Corresponding Author: {\tt\small farzan.memarian@utexas.edu}}%
\thanks{$^1$ Farzan Memarian is with the Oden Institute for Computational Engineering and Sciences, University of Texas at Austin, TX, USA. Wonjoon Goo, Rudolf Lioutikov and Scott Niekum are with the Department of Computer Science, University of Texas at Austin, TX, USA. Ufuk Topcu is with the Department of Aerospace Engineering and  Engineering  Mechanics, University of Texas at Austin, TX, USA.}

}

\begin{document}

\maketitle
\thispagestyle{fancy}
\pagestyle{empty}

\begin{abstract}

We introduce Self-supervised Online Reward Shaping (SORS), which aims to improve the sample efficiency of any RL algorithm in sparse-reward environments by automatically densifying rewards.
The proposed framework alternates between classification-based reward inference and policy update steps—the original sparse reward provides a self-supervisory signal for reward inference by ranking trajectories that the agent observes, while the policy update is performed with the newly inferred, typically dense reward function.
We introduce theory that shows that, under certain conditions, this alteration of the reward function will not change the optimal policy of the original MDP, while potentially increasing learning speed significantly.
Experimental results on several sparse-reward environments demonstrate that, across multiple domains, the proposed algorithm is not only significantly more sample efficient than a standard RL baseline using sparse rewards, but, at times, also achieves similar sample efficiency compared to when hand-designed dense reward functions are used.

\end{abstract}

\section{Introduction}

While reinforcement learning (RL) algorithms have achieved tremendous success in many tasks ranging from Atari games \cite{mnih2015human, mnih2016asynchronous, van2015deep} to robotics control problems \cite{gu2017deep, kober2013reinforcement, levine2016end}, they often struggle in environments with sparse rewards. 
In dense reward settings, the agent receives diverse rewards in most states, e.g., a reward proportional to distance to the goal, rather than a constant reward everywhere but the goal. Such dense rewards lead to frequent updates that quickly allow the agent to differentiate good states from bad ones.

%%%% Motivation for suggesting a new framework.
%%%% Reward designing is hard / reward learning and shaping is bad.
Unfortunately, designing a good, dense reward function is known to be a difficult task \cite{abbeel2004apprenticeship,sutton2018reinforcementlearning}, especially for non-experts. 
In addition, RL approaches can easily exploit badly designed rewards, get stuck in local optima and induce behavior that the designer did not intend \cite{hadfield2017inverse}. 
% Dense rewards allow easier credit assignment and exploration.
In contrast, goal-based sparse rewards are appealing since they do not suffer from the reward exploitation (commonly known as reward hacking) problem to the same extent.
However, sparse rewards only provide rewards for few select states.
Reward sparseness complicates the temporal credit assignment problem significantly and negatively impacts the overall learning process.
Reward shaping is a commonly used approach to speed up RL in environments with sparse rewards \cite{ng1999policy,wiewiora2003principled,brys2015reinforcement}. 
However, altering the ground-truth reward can potentially change the optimal policy and, hence, induce undesired behavior.

In this paper, we propose a novel RL framework that efficiently learns a policy for sparse-reward environments by training on dense rewards that are inferred in a self-supervised manner.
Our framework---\textbf{S}elf-supervised \textbf{O}nline \textbf{R}eward \textbf{S}haping (SORS)---can speed up the policy learning process without requiring any domain knowledge or external supervision, and the proposed framework is compatible with any existing RL algorithm.

SORS alternates between updating the policy using an RL algorithm of choice and inferring a dense reward function from past observations. 
It infers a reward using a classification-based reward inference algorithm, T-REX \cite{brown2019extrapolating}. 
However, unlike T-REX, instead of requiring manual rankings over the trajectories, SORS uses the sparse reward as a self-supervised learning signal to rank the trajectories generated by the agent during learning.

We justify the rationale behind the reward inference performed by SORS based on the following insight: Any reward function induces a total order over the trajectory space by means of the discounted return it assigns to trajectories. 
We then provide a theorem that indicates any two reward functions that induce the same \textit{total order} over the trajectory space, induce identical sets of optimal policies under mild assumptions on the dynamics of the environment. 
The objective function that SORS optimizes for reward inference encourages the dense reward function to induce the same total order as the sparse reward over the trajectory space.

%%%% Brief Intro to results
Our empirical results on several sparse reward MuJoCo \cite{todorov2012mujoco} locomotion tasks show that SORS can significantly improve the sample efficiency of the state-of-the-art baseline algorithm, namely Soft-Actor-Critic (SAC). 
SORS even achieves comparable sample efficiency to a baseline that uses a hand-designed dense reward function.

%%%% Summary of the contribution
\noindent We make the following contributions: 
\begin{itemize}
\item We propose a novel reward shaping algorithm, SORS, that pairs with any existing RL algorithm, performs self-supervised online reward shaping, and can improve the sample efficiency of the RL algorithm in sparse-reward environments.
\item We provide theoretical justification for our approach by showing a sufficient condition for two reward functions to share the same set of optimal policies. 
We use this condition to show that, under some assumptions, replacing the ground-truth sparse reward function with the inferred shaped reward function does not alter the optimal policies.
\item We empirically demonstrate that the proposed method converges significantly faster than a standard baseline RL algorithm, namely Soft Actor-Critic (SAC) \cite{haarnoja2018soft} for several sparse-reward MuJoCo locomotion tasks.
\end{itemize}

\section{Related work}
\label{sec:related-work}

\subsection{Reward Shaping}
Reward shaping is a method to incorporate domain knowledge to densify reward functions.
Typically, the goal of reward shaping is to speed up learning and overcoming the challenges of exploration and credit assignment when the environment only returns a sparse, uninformative, or delayed reward. 

In one of the seminal works on reward shaping \cite{ng1999policy}, the authors study the forms of shaped rewards which induce the same optimal policy as the ground-truth reward function.
Specifically, they proved that the so-called \textit{potential-based} reward shaping is guaranteed not to alter the optimal policy.
The only requirement is that the potential function needs to be a function of states. 
While they provide one specific form for reward shaping without altering the optimal policy of the MDP, they do not provide any practical algorithm for acquiring a potential function that can improve the learning of optimal behavior.
They argue that the optimal state value function is a good shaping potential, but this insight is not helpful in practice, as the goal of RL is finding the optimal value function is the goal of RL and we do not have it a priori. 
In this work, we propose an alternative reward shaping framework in which we replace the original reward function with another shaped reward function which is updated online as the RL agent interacts with the environment. 
Our reward shaping approach does not require any human guidance or extra information.

Devlin et al. \cite{devlin2012dynamic} build on potential-based shaping \cite{ng1999policy} to prove that dynamic shaping of the reward function does not change the optimal policy, provided that we use the potential-based shaping framework. 
Other researchers \cite{wiewiora2003principled} have extended potential-based shaping \cite{ng1999policy} to potential functions that are functions of state and action pairs rather than states alone. 
They propose two methods for providing potential-based advice, namely, look-ahead advice, and look-back advice. 

In another interesting work on reward shaping \cite{trott2019keeping}, the authors propose a new RL objective which uses a distance-to-goal shaped reward function.
They unroll the policy to produce pairs of trajectories from each starting point and use the difference between the two rollouts to discover and avoid local optima. 
Unlike their work, we do not need to alter the way the base RL algorithm collects experiences. 
Moreover, we do not rely on using a distance-to-goal shaped reward function, instead we learn a dense reward function which asymptotically leads to optimal policies that equivalent to those of the original sparse reward. 

There is prior work on automatic reward shaping \cite{zou2019reward}, where they propose reward shaping via meta-learning. 
Their method can automatically learn an efficient reward shaping for new tasks, assuming the state space is shared among the meta-learning tasks.  
This work differs from ours in that it is in the context of meta learning, whereas our automatic reward shaping algorithm works even for a single task, and we do not need to train our model on a library of prior tasks. 

Brys et al. propose a method to use expert demonstrations to accelerate RL by biasing the exploration through reward shaping \cite{brys2015reinforcement},. 
They propose a potential function which is higher for state-action pairs similar to those seen in the demonstrations and low for dissimilar state-action pairs. 
Another related work studies online learning of intrinsic reward functions as a way to improve RL algorithms \cite{pathak2017curiosity}.

\subsection{Sparse Rewards}

RL in sparse-reward environments has been tackled in various ways.
For instance, the authors of \cite{riedmiller2018learning} address sparse-reward environments that can be de-composed into smaller subtasks. 
They learn a high-level scheduler and several auxiliary policies and show that this leads to improved exploration. 
Their algorithm learns to provide internal auxiliary sparse rewards in addition to the original sparse reward. 
Our algorithm is different from this line of work as our algorithm works for singular tasks, and we do not use any hierarchy of decision making. 
We learn a dense reward which assigns a reward to every individual state, rather than merely providing an auxiliary reward on selected states. 

Other related work \cite{memarian2020active} on learning from sparse rewards proposes a method to learn a temporally extended episodic task composed of several subtasks where the environment returns a sparse reward only at the end of the episodes.
Using the environment’s sparse feedback and queries from a demonstrator, they learn the high-level task structure in the form of a deterministic finite state automaton, and then use the learned task structure in an inverse reinforcement learning (IRL) framework to infer a dense reward function for each subtask. 
Our work differs in that we do not rely on an expert to provide demonstrations and instead we learn to shape the sparse reward relying only on the environment's sparse feedback.

\subsection{Learning a Reward Function From Preference/Ranking}
Several prior works have studied the problem of inferring a reward function from human preferences or rankings over demonstrations. 
One early work on learning from preferences \cite{christiano2017deep} proposes an active learning approach to infer a reward function that encodes the human’s preferences. 
They train a policy and a reward network simultaneously. 
At each iteration, they use the policy to produce pairs of trajectories and then query the human for their preference over the pair of trajectories and use these preferences to improve the reward by minimizing a preference-based loss function. 
They then updated the policy based on the improved reward. 
In other work \cite{ibarz2018reward}, Ibarz et al. extend the work Christiano et al. \cite{christiano2017deep} to use an initial set of demonstrations to pre-train the policy, rather than start training from a random policy. 
Brown et al. introduce the T-REX algorithm, which infers a reward function from a given set of ranked demonstrations \cite{brown2019extrapolating},. 
Their algorithm samples pairs of demonstrations from this initial set of demonstrations and uses the ranking to label which demonstration is preferred in a given pair. 
It then uses a binary classification loss over these preferences to update the reward function.
They show their algorithm learns reward functions that, when optimized for a policy, often exceed the performance of the best demonstrations.

We use an adaptation of the T-REX algorithm for the reward inference part of our algorithm.  However, our work is different from the above works in two ways. First, we do not need an initial set of demonstrations. Second, our algorithm does not require a human in the loop---instead we leverage the environment's sparse feedback to rank the collected trajectories and then use the set of ranked trajectories for inferring a dense reward function to accelerate policy learning.  

In another work, Brown et al. propose an algorithm to infer a reward from a set of sub-optimal demonstrations that are not ranked by an expert \cite{brown2019ranking}. 
Using the set of demonstrations, they perform behavioral cloning to learn a policy. 
They then inject noise in the policy to produce various qualities of trajectories and rank the trajectories based on the level of noise used in producing them. 
Then they proceed to learn a reward from the set of ranked trajectories. 
Our work differs, in that our self-supervisory signal comes from a known sparse reward signal on agent-collected trajectories, and
our objective is to use the learned reward function as a way to accelerate policy learning, rather than imitate demonstrations.

\section{Background and Preliminaries} \label{sec:background}

\subsection{Reinforcement Learning}
A Markov decision process (MDP) is defined as $\mathcal{M} = \langle S, A,T, r, \gamma \rangle$, in which $S$ is the state space, $A$ is the action space, $T: S \times A \to \mathcal{P}(S)$ is the transition dynamics which maps any given state and action pair into a probability distribution over the next state, $r: S \times A \to \mathbb{R}$ is the reward function, and $\gamma$ is the discount factor. 
At each discrete time step, the MDP is in a state $s$, the agent takes an action $a$, and as a result, the MDP transitions into a new state, $s' \sim T(s, a)$, and the agent receives a scalar valued reward $r(s,a,s')$.
A policy $\pi(a|s): S \to \mathcal{P}(A)$ is defined as a probability distribution over actions at any given state $s$. 
Given a policy $\pi$, we have the following definitions:
\begin{align*}
    &Q^\pi(s,a) = r(s,a) + \gamma \, \mathbb{E}_{a' \sim \pi(a'|s')}  \mathbb{E}_{s'\sim T(s,a)}  [Q(s',a')]  \\
    &V^\pi(s) = \mathbb{E}_{a \sim \pi(a|s)} [Q^\pi(s,a)]
    % &A^\pi(s,a) = Q^\pi(s,a) - V^\pi(s)
\end{align*}
where $ Q^\pi(s,a)$, $V^\pi(s)$ are respectively the action-value function and the state value function for the policy $\pi$.
The goal of RL is to find a policy with maximal value function at each state, or find the maximal value function directly. 
A trajectory $\tau= \{s_t,a_t \}_{t=1}^{|\tau|}$ is a sequence of state action pairs obtained by running a policy on the MDP, where subscript $t$ is the time index of the trajectory, i.e., each trajectory starts from $(s_1, a_1)$. 
We define the discounted return of a trajectory according to reward function $r$ as: $R_r(\tau) := \sum_{t=1}^{|\tau|} \gamma^{t-1} r(s_t,a_t)$, where $(s_t,a_t)$ is the state and action pair of the trajectory $\tau$ at time $t$.

\subsection{Reward Shaping}
\label{sec:preference-oracle}

Given an MDP with a reward function $r(s,a)$, reward shaping is the process of replacing the original reward with another reward function, or augmenting the original reward function with an auxiliary reward function $F(s,a):S \times A \to \mathbb{R}$ to create a new reward function\cite{ng1999policy}; Concretely, $r_{sh}(s,a) = r_{new}(s,a) ~~ \text{or }~ r_{sh}(s,a) = r(s,a) + F(s,a)$,  where $r_{sh}(s,a)$ is the shaped reward. 
While the goal of reward shaping is to speed up RL, in general, a shaped reward could induce a different optimal policy than the original reward.

\section{Preference Oracle and Equivalency of Reward Functions}

Consider a reward-free MDP $\mathcal{M} = \langle S,A,T,\gamma \rangle$, and a  \textit{preference oracle} which is a binary relation $\leq_{p*}$ that defines a total order on the set of all trajectories sampled from the MDP. 
We can order all possible trajectories based on the total order defined by the oracle:
\begin{equation*}
\tau_1 \leq_{p*} \tau_2 \leq_{p*} \cdots \leq_{p*} \tau_k \leq_{p*} \cdots.
\end{equation*}
Note that any deterministic reward function $r(s,a)$ can serve as a preference oracle via the discounted return $R_r$ under that reward function:
\begin{equation*}
\tau_i \leq_{r} \tau_j \Leftrightarrow R_r(\tau_i) \leq R_r(\tau_j).
\end{equation*}
where $\leq_{r}$ is the binary relation defined by reward function $r$. Using the notion of total order, we will define a set of reward functions that share the same set of optimal policies; specifically, we will prove that two reward functions that produce the same total order will also yield the same set of optimal policies under deterministic transition dynamics. 
We begin by formally defining the total order equivalency between two reward functions.
\begin{definition}[Total order equivalency]\label{def:total-order-equivalency}
For a given reward-free MDP $\mathcal{M} = \langle S,A,T,\gamma \rangle$ with possible trajectories $\mathcal{T} = (S \times A)^{+}$, the total order equivalency of reward functions $r_1$ and $r_2$ is defined as 
\begin{equation*}
    r_1 \equiv r_2\ \text{iff}\ \tau_i \leq_{r_1} \tau_j \Leftrightarrow \tau_i \leq_{r_2} \tau_j \ \forall \tau_i, \tau_j \in \mathcal{T}.
\end{equation*}
\end{definition}

\begin{theorem}
\label{theorem:policy-equivalency}
Given a deterministic reward-free MDP $\mathcal{M} = \langle S,A,T,\gamma \rangle$, if two reward functions $r$ and $r'$ are total order equivalent, they will induce the same set of optimal policies, i.e., $r \equiv r' \implies \{ \pi^*_{r}(s) \} = \{ \pi^*_{r'}(s) \}$, where $ \{ \pi^*_{r} \}$ and $\{ \pi^*_{r'} \}$ are the sets of optimal policies induced by reward functions $r$ and $r'$ respectively.
\end{theorem}

\begin{proof}
The state-action value function of a policy $\pi$ at a given state and action pair $s,a$ is defined as:
\begin{align*}
Q^{\pi}(s,a) = \mathbb{E}_{\pi,T}\left[R_r\left(\tau_{s,a}\right)\right]
\end{align*}
which is equal to the expected return over all trajectories $\tau_{s,a} \in \mathcal{T}_{s,a}$ that start with action $a$ at state $s$ and follow policy $\pi$ under the transition dynamics $T$.
Given the optimal state-action value function $Q_r^*$ for reward function $r$, an optimal policy under the reward function $r$ is derived as $\pi^*_r(s) = \argmax_a Q^*_r(s,a)$ \footnote{There can be more than one optimal policy corresponding to a given optimal Q-function. For example, if multiple actions maximize the Q-function at a given state.}. 
Following these definitions, it is clear that any action chosen by an optimal policy will yield the highest possible Q value, i.e.,
\begin{align*}
Q^*_r(s,\pi_r^*(s)) \geq Q^*_r(s,b) \quad\forall b \in A.
\end{align*} 

For MDPs with deterministic dynamics, an optimal policy under a reward function $r$ will induce a set of optimal trajectories starting from any state-action pair, where all the optimal trajectories receive equal returns from reward function $r$. 
If the policy is deterministic, the set of optimal trajectories will include only one member. 
Hence, for deterministic MDPs the optimal Q-function for state-action pair $(s,a)$ and an optimal trajectory starting from the same pair are
\begin{align*}
    Q^*_r(s,a) &=  \max_{\tau \in \mathcal{T}_{s,a}} R_r(\tau)\text{, and}\\
    \tau^*(s,a) &= \argmax_{\tau \in \mathcal{T}_{s,a}} R_r(\tau).
\end{align*}

Using the total order relation $\leq_{r}$ induced by the reward function $r$, and the equivalence between $r$ and $r'$, we conclude that the two reward functions share the same set of optimal policies. 
In other words, if $\pi_r^*$ is an optimal policy under reward function $r$, it is an optimal policy under reward function $r'$ as well:
\begin{align*}
    && \forall b \in A,\ Q^*_r(s,b) &\leq Q^*_r(s,\pi^*_r(s)) \\
    &\Leftrightarrow &\max_b \max_{\tau \in \mathcal{T}_{s,b}} R_r(\tau) &\leq \max_{\tau \in \mathcal{T}_{s,\pi^*_r(s)}} R_r(\tau)\\
    &\Leftrightarrow &\max_b \tau^*(s,b) &\leq_{r} \tau^*(s,\pi^*_r(s))\\
    &\Leftrightarrow &\max_b \tau^*(s,b) &\leq_{r'} \tau^*(s,\pi^*_r(s)) \tag*{($\because r \equiv r'$)}\\
    &\Leftrightarrow &\max_b \max_{\tau \in \mathcal{T}_{s,b}} R_{r'}(\tau) &\leq \max_{\tau \in \mathcal{T}_{s,\pi^*_r(s)}} R_{r'}(\tau)\\
    &\Leftrightarrow &\forall b \in A,\ Q^*_{r'}(s,b) &\leq Q^*_{r'}(s,\pi^*_r(s))\\
    % \Leftrightarrow& \pi^*_r(s) = \argmax_a Q^*_{r'}(s,a)\\
    % \Leftrightarrow& \pi^*_r(s) = \pi^*_{r'}(s).
    &\Leftrightarrow &\span \text{$\pi^*_r(s)$ is an optimal policy under $r'$.} \quad\quad\quad\qedhere
\end{align*}
\end{proof}

Theorem~\ref{theorem:policy-equivalency} suggests that a set of optimal policies is uniquely defined by the total order, and there are potentially infinitely many reward functions that share the same set of optimal policies. 
Among these reward functions, some are preferable with respect to efficiency of policy learning. 
While sparse rewards are hard to learn from due to the credit assignment difficulty, a more informative reward function (potentially dense) can exist that shares the same set of optimal policies and is much easier to learn from.
This implication is consistent with the optimal reward problem \cite{sorg2011optimal} and reward shaping \cite{ng1999policy}.

While the specification of a set of reward functions that share the same optimal policy has been studied \cite{ng2000algorithms, brown2019machineteaching}, the proposed theorem is more general in that we do not assume any restriction on the reward function space. 
In \cite{ng2000algorithms}, a behavior equivalence class (BEC) is defined across reward functions that share the same feature vector extractor $\phi(s,a)$, so the reward function space is restricted to the span of the feature vector space. 
The BEC can be very small if the feature space is not diverse enough and defining good features a priori requires external knowledge or a well-designed loss function \cite{brown2020safe}. 
By contrast, our theory does not have any restrictions on the form of reward function, so our notion of equivalence can contain a larger reward function set than BEC.

While the preference oracle can define the optimal behavior that we want to induce, it is unreasonable to assume that we have such an oracle at hand, since it requires a total order over all possible trajectories. 
Instead, previous methods working with orderings between trajectories assume external human input in an online \cite{christiano2017deep} or offline manner \cite{brown2019extrapolating}, with a human preference oracle. 
While we use the same loss function as these approaches, we focus on the reward shaping problem in the sparse reward scenario for which we have a coarse notion of task progress or success. 
Specifically, we try to infer a new, potentially dense reward function that satisfies the order constraints imposed by the sparse reward function and replace the original reward with the inferred reward function to improve the sample efficiency of policy learning. The detailed explanation of the method is presented in the next section.

\section{Method}\label{sec:method}

We tackle the problem of RL in sparse-reward environments. 
The key idea is to infer a dense reward function that shares the same set of optimal policies with the sparse reward, and use the inferred reward function for policy learning to foster faster, sample efficient learning. 
We call the proposed RL framework \textbf{S}elf-supervised \textbf{O}nline \textbf{R}eward \textbf{S}haping (SORS). 

% High-level algorithm description -- Two-step approach
SORS alternates between online reward shaping and reinforcement learning with the inferred reward function. 
During the online reward shaping, a potentially dense reward function is trained with a loss function that encourages the inferred reward to create the same total order over trajectories  as the sparse reward. 
During reinforcement learning, the policy is trained with the inferred reward function and new trajectories are collected in the process. 
Since SORS can work with any RL algorithm, we mainly focus on discussing the online reward shaping module. 
The overall framework with an off-policy RL back-end is described in Algorithm.~\ref{alg:self-supervised-online-reward-shaping}.

% \subsection{Online Reward Shaping}

We train a parameterized reward function $r_\theta$ by encouraging it to satisfy the order constraints imposed by the ground-truth sparse reward function $r_s$. 
Specifically, we train the reward function with a binary classification loss over pairs of trajectories sampled from the trajectory buffer $\mathcal{D}_\tau$ that saves every observed trajectory during reinforcement learning. 
The loss function is formally defined as:
\begin{equation}\label{eq:reward-loss}
    \begin{aligned}
    L(\theta;\mathcal{D}_\tau) = &- \sum_{(\tau_i , \tau_j) \sim \mathcal{D}_\tau}
\bigg[ \mathbb{I}(\tau_i \leq_{r_s} \tau_j) \log P(\tau_i \prec \tau_j)\\
    &+ (1 - \mathbb{I}(\tau_i \leq_{r_s} \tau_j)) \log P(\tau_i \succ \tau_j) \bigg],
    \end{aligned}
\end{equation}
where $\mathbb{I}(.)$ is the indicator function that evaluates to one if the condition inside it is True, and evaluates to zero otherwise. $P(\tau_i \prec \tau_j)$ is defined as:
\begin{equation}
    P(\tau_i \succ \tau_j) = \frac{\exp(R_{r_\theta}(\tau_i))}{\exp(R_{r_\theta}(\tau_i)) + \exp(R_{r_\theta}(\tau_j))}.
\end{equation}
This same loss function has been used in other work to train a reward function with given pair-wise preference over trajectories, since the loss function encourages the learned reward to assign a higher return to the preferred trajectory \cite{christiano2017deep,brown2019extrapolating}. 
While our final goal is not just to infer a reward based on the pairwise preferences, but learning a reward function that satisfies the total order constraints generated by the ground-truth sparse reward, we empirically find that pairwise preference-based loss can enforce a total order comparable to the ground-truth total order. 
We leave the use of recently proposed ranking loss \cite{rankingloss} that considers the total order as a future work.

Note that SORS does not make use of any external information in addition to what an ordinary RL algorithm requires; the framework receives the exact same observations and rewards from the environment as a baseline RL algorithm would, and it performs the online reward shaping in a \textit{self-supervised} manner. 
Although SORS does not use any extra information, we hypothesize that the additional reward shaping module may improves learning since (1) we can leverage a deep neural network in inferring the relevant features that may be infeasible for a human to define when writing down a reward function, and (2) SORS performs credit assignment not only when learning a value function / policy (as in standard RL), but also by inferring a new reward function from the automatically-ranked trajectories that it collects.

Another way that we could have used the learned dense reward is as a shaping potential, ensuring convergence to the same set of optimal policies \cite{ng1999policy}. 
As discussed in related works, potential-based reward shaping is a theoretically sound way of shaping the reward functions while ensuring that the optimal policy is maintained \cite{ng1999policy}.
However, the effects of potential-based reward shaping can be ``learned away" over time, as they are equivalent to value function initialization \cite{wiewiora2003potential}. Hence, we choose to use the learned dense reward to replace the environment's sparse reward rather than using it as a shaping potential.

\begin{algorithm}[t]

\caption{SORS RL Framework (w/ off-policy RL algo.)} 
\label{alg:self-supervised-online-reward-shaping}

\begin{algorithmic}[1]
\STATE {\bfseries Input:} An environment with sparse reward $r_s(s,a)$. A base RL algorithm of choice (SAC in this work).

\STATE {\bfseries Output:} $\theta$: Parameters  of the dense reward network $r_\theta(s, a)$. $\phi$: Parameters of the policy network $\pi_\phi(a|s)$.

\STATE {\bfseries Hyper-parameters: } $N$: Total number of environment interactions. $P_r, N_r$: Reward update period and number of reward updates for every period. $P_p, N_p$: RL update period and number of RL updates for every period

\STATE Initialize $\theta$ and $\phi$, initialize the trajectory buffer $\mathcal{D}_\tau$ to an empty set.

\STATE \textit{// Collect Initial Trajectories}
\STATE Run a random policy and fill up the trajectory buffer

\FOR{$i = 1\ \dots\ N$}
    \STATE \textit{// Gather Experience}
    \STATE Execute the current stochastic policy and append the transition tuples to the trajectory buffer $\mathcal{D}_\tau$.
    \STATE Replace old trajectories if buffer is full.
    
    \STATE \textit{// Dynamic Reward Shaping Module}
    \IF{$i\ \text{mod}\ P_r = 0$}
        \FOR{$N_r$ iterations}
            \STATE Update $\theta$ with respect to the loss defined in Eq.\ref{eq:reward-loss} with trajectory pairs sampled from $\mathcal{D}_\tau$.
        \ENDFOR
    \ENDIF
    
    \STATE \textit{// Reinforcement Learning Module}
    \IF{$i\ \text{mod}\ P_p = 0$}
        \FOR{$N_p$ iterations}
            \STATE Update $\phi$ according to the latest shaped reward $r_\theta(s,a)$ using the base RL algorithm. 
        \ENDFOR
    \ENDIF
\ENDFOR
\end{algorithmic}
\end{algorithm}

\section{Experiments}
\label{sec:experiments}
\begin{figure*}[ht]
    \centering
    \begin{subfigure}[c]{0.32\linewidth}
        \includegraphics[width=\textwidth]{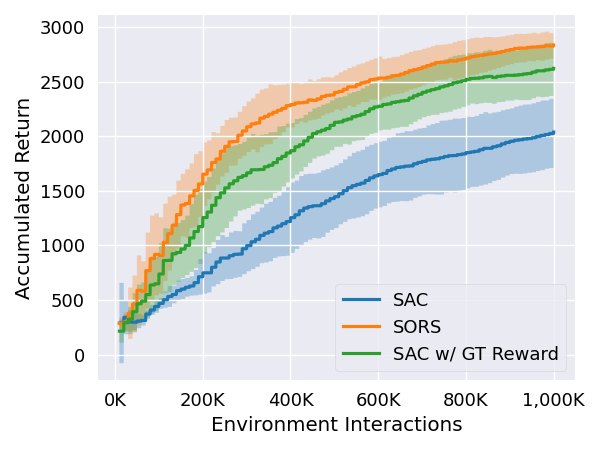}
        \caption{Delayed Hopper}
    \end{subfigure}
    \begin{subfigure}[c]{0.32\linewidth}
        \includegraphics[width=\textwidth]{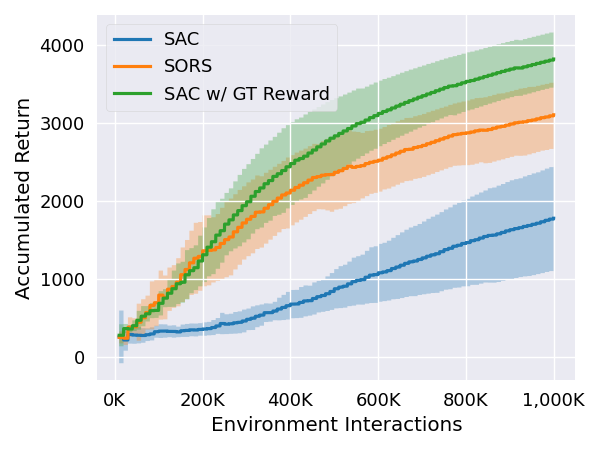}
        \caption{Delayed Walker2D}
    \end{subfigure}
    \begin{subfigure}[c]{0.32\linewidth}
        \includegraphics[width=\textwidth]{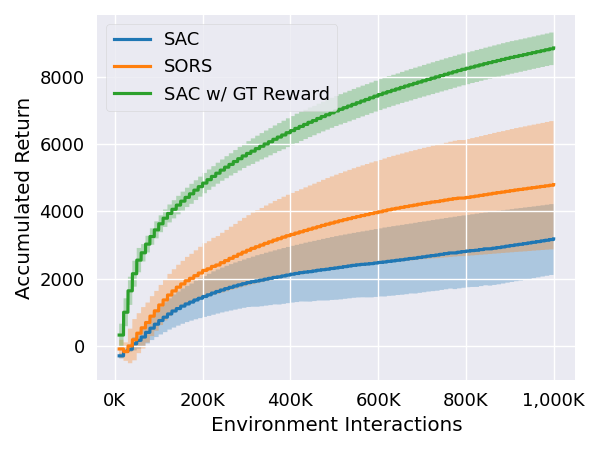}
        \caption{Delayed HalfCheetah}
    \end{subfigure}
    \begin{subfigure}[c]{0.32\linewidth}
        \includegraphics[width=\textwidth]{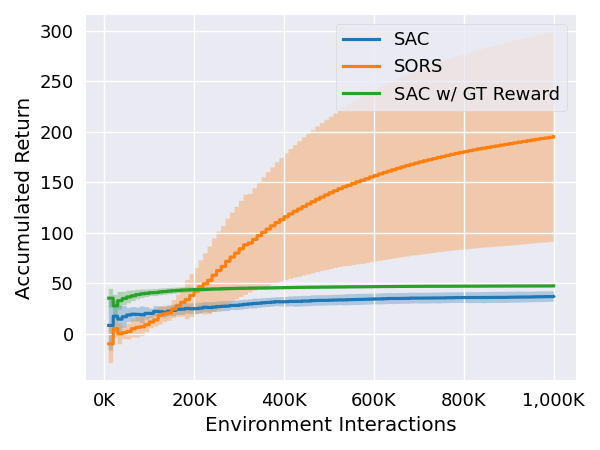}
        \caption{Delayed Swimmer}
    \end{subfigure}
    \begin{subfigure}[c]{0.32\linewidth}
        \includegraphics[width=\textwidth]{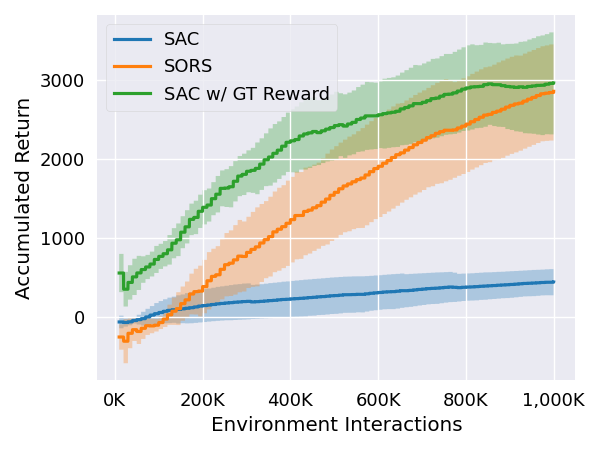}
        \caption{Delayed Ant}
    \end{subfigure}
    \begin{subfigure}[c]{0.32\linewidth}
        \includegraphics[width=\textwidth]{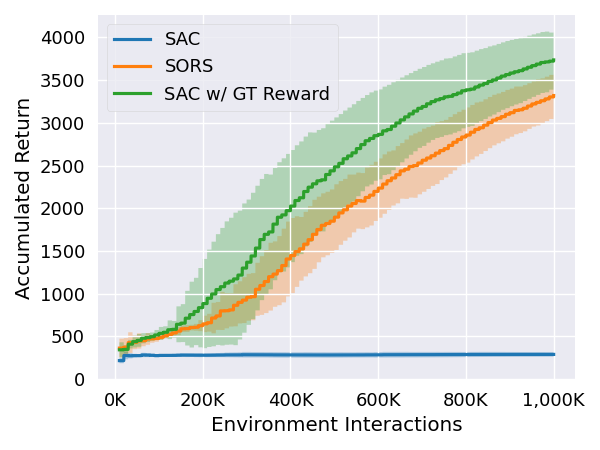}
        \caption{Delayed Humanoid}
    \end{subfigure}
    
    \caption{Learning curves of SAC (without reward shaping, blue line), SORS (with reward shaping, orange line), and SAC with a hand-designed dense reward function (green line). 
    A trajectory is generated  every 10,000 steps with the policy at that step, and the return of the generated trajectory is reported. 
    We smooth the curve with an exponential moving average with a half-life time of 2,000. The results are averaged over 5 different random seeds, and the shaded area represents standard deviation.}
    \label{fig:experiments}
\end{figure*}

We aim to study the following questions: (1) Does the inferred dense reward function improve the sample efficiency of the base RL algorithm? (2) Will the inferred dense reward function induce the same policy as the one induced by the ground-truth reward function?

Reward shaping is particularly helpful when the ground-truth reward is sparse or otherwise hard to learn from. Hence, we test SORS on delayed MuJoCo environments \cite{oh2018self, guo2018generative} in which rewards are accumulated for a given number of time steps (20 time steps) and provided only at the end of these periods or the end of the episode, whichever comes first. 
We use 6 MuJoCo locomotion tasks, namely Hopper, Walker2d, HalfCheetah, Swimmer, Ant, and Humanoid whose observation and action space range from small (8 and 2 for Swimmer respectively) to large (376 and 17 for Humanoid respectively). The code is available online\footnote{Code: \url{https://github.com/hiwonjoon/IROS2021_SORS}}.

We choose the Soft-Actor-Critic (SAC) algorithm as the back-end RL algorithm, and we compare the training progress of the proposed method against a baseline that trains a policy with (1) the delayed reward or (2) the ground-truth dense reward provided by the MuJoCo environment. 
Note that the SAC method is a strong baseline, which is better than or comparable to other regularized RL algorithms \cite{oh2018self} on the MuJoCo environments, and hence we omit other baselines. 

For SAC implementation, both the policy and the Q-functions are modeled by fully connected neural networks with 3 hidden layers, where each layer is of size $256$ and is followed by ReLU non-linearity.
The stochastic policy is modeled by a diagonal multivariate normal distribution and its parameters (mean and covariance) are generated via the policy network. 
We use the same techniques introduced in SAC, such as dual Q-training, use of slowly updated target Q network, and dynamically adjusted entropy regularization coefficient. 
The Q-function and the policy are updated for 50 stochastic gradient descent steps with a mini-batch of size $100$ using Adam optimizer with a learning rate of $3e-4$ after every $50$ interactions with the environment.

The architecture of the neural network modeling the dense reward function is as follows: 
$3$ fully connected hidden layers of size $256$, followed by a fully connected hidden layer of size $4$.
The output of the network up to this point will be a $4$ dimensional feature vector.
All the hidden layers are followed by $tanh$ non-linearity.
The final output of the network is computed by applying a weight vector $w$ to the $4$ dimensional feature vector. 
We enforce the condition $||w||_2=1$ to limit the scale of the reward. 
Both the neural network parameters and the reward weight vector $w$ are trained together by minimizing the loss function given in Eq.~\ref{eq:reward-loss}. 
To reduce the variance between runs and improve the stability of our method, we train an ensemble of 4 reward networks with different initializations and take the average of their outputs to produce the final reward.
At the beginning of training SORS, we run a random policy for $2000$ steps to collect an initial set of trajectories. 
During the rest of the run, we call the dense reward learning module after every $1,000$ environment steps, and perform $100$ stochastic gradient descent steps using a mini-batch of size 10 trajectory pairs.
We keep training both the baseline and SORS until the agent interacts with the environment for $10^6$ steps. 

Figure.~\ref{fig:experiments} shows the comparison between SORS and the baselines on several environments. 
In all $6$ delayed MuJoCo environments, SORS learns faster than the baseline trained on the delayed reward.
Moreover, SORS shows similar sample efficiency and asymptotic performance to the baseline trained with the ground-truth dense reward function on all environments except HalfCheetah. 
This implies that the proposed method can successfully densify the reward function. 
Swimmer is an example of an environment where the policy trained with our method converged to a better policy than the baseline that uses the original dense reward of the environment. 
The results on Swimmer support our hypothesis on the existence of an informative reward function that potentially fosters faster reinforcement learning than the sparse reward, or even the ground truth dense reward function.

\section{Conclusion}

We propose a novel reward shaping method, called SORS, which aims to infer a reward function that satisfies the preference constraints given by the original sparse reward function. 
Since the constraints can be automatically generated by observing the return of the experienced trajectories according to the sparse reward, the proposed algorithm is fully self-supervised.
Our experiments show that SORS enables faster, more sample efficient reinforcement learning by generating a dense reward function that induces a policy with strong performance with respect to the original sparse reward. 
Our experiments show that it is easier to learn from the learned dense reward, as it provides more immediate feedback, even though the assumptions needed for a theoretical guarantee of leaving the set of optimal policies unchanged are not strictly met. 
Along these lines, future work may include an investigation of how to provide guarantees with respect to optimality in the original MDP under weaker assumptions than what we have provided in this work.

\section{Acknowledgement}
\small This work was partly supported through the following grants: ARL W911NF2020132, NSF 1652113 and ARL ACC-APG-RTP W911NF1920333. 
Also, this work has taken place in part in the Personal Autonomous Robotics Lab (PeARL) at The University of Texas at Austin. PeARL research is supported in part by the NSF (IIS-1724157, IIS-1638107, IIS-1749204, IIS-1925082), ONR (N00014-18-2243), AFOSR (FA9550-20-1-0077), and ARO (78372-CS). This research was also sponsored by the Army Research Office under Cooperative Agreement Number W911NF-19-2-0333. The views and conclusions contained in this document are those of the authors and should not be interpreted as representing the official policies, either expressed or implied, of the Army Research Office or the U.S. Government. The U.S. Government is authorized to reproduce and distribute reprints for Government purposes notwithstanding any copyright notation herein.

\bibliography{main}
\bibliographystyle{IEEEtran}

\end{document}